\def\year{2016}
\documentclass[letterpaper]{article}
\usepackage{aaai16}
\usepackage{times}
\usepackage{helvet}
\usepackage{courier}

\usepackage{amsfonts}
\usepackage{amssymb}
\usepackage{enumitem}
\usepackage{mathtools}
\usepackage{array}
\usepackage{graphicx}
\usepackage{tikz}
\usetikzlibrary{decorations.pathmorphing}
\tikzset{attack/.style={-latex, thick, decorate, decoration={zigzag, amplitude=1, post length=2pt}}}
\tikzset{symmetric attack/.style={latex-latex, thick, decorate, decoration={zigzag, amplitude=1, post length=2pt, pre length=2pt}}}
\tikzset{self-attack/.style={thick, decorate, decoration={zigzag, amplitude=1}}}
\tikzset{pattack/.style={-latex, thick}}
\tikzset{symmetric pattack/.style={latex-latex, thick}}
\tikzset{self-pattack/.style={}}
\tikzset{normal attack/.style={-latex, dashed, thick}}
\tikzset{symmetric normal attack/.style={latex-latex, dashed, thick}}
\tikzset{normal self-attack/.style={dashed, thick}}
\tikzset{reverse attack/.style={-latex, dotted, thick}}
\tikzset{symmetric reverse attack/.style={latex-latex, dotted, thick}}
\tikzset{reverse self-attack/.style={dotted, thick}}
\tikzset{grayed/.style={fill=gray, fill opacity=0.3}}

\usepackage{amsthm}
\theoremstyle{plain}
\newtheorem{theorem}{Theorem}
\newtheorem{lemma}[theorem]{Lemma}
\newtheorem{proposition}[theorem]{Proposition}
\newtheorem{corollary}[theorem]{Corollary}
\theoremstyle{definition}
\newtheorem{definition}{Definition}
\newtheorem{principle}[definition]{Definition}
\newtheorem{axiom}[definition]{Axiom}
\newtheorem{example}[definition]{Example}

\def\abap{ABA$^+$}
\def\aspicp{ASPIC$^+$}
\def\semantics{stable, complete, preferred, grounded, ideal}

\def\contraposition{the Axiom of Contraposition}
\def\wcp{the Axiom of Weak Contraposition}

\newcommand*{\cut}{\texttt{CUT}}
\newcommand*{\scut}{\texttt{STRICT CUT}}
\newcommand*{\scuts}{\texttt{SCEPTICAL STRICT CUT}}
\newcommand*{\scutc}{\texttt{CREDULOUS STRICT CUT}}
\newcommand*{\acut}{\texttt{ASM CUT}}
\newcommand*{\acuts}{\texttt{SCEPTICAL ASM CUT}}
\newcommand*{\acutc}{\texttt{CREDULOUS ASM CUT}}
\newcommand*{\mon}{\texttt{MON}}
\newcommand*{\smon}{\texttt{STRICT MON}}
\newcommand*{\smons}{\texttt{SCEPTICAL STRICT MON}}
\newcommand*{\smonc}{\texttt{CREDULOUS STRICT MON}}
\newcommand*{\amon}{\texttt{ASM MON}}
\newcommand*{\amons}{\texttt{SCEPTICAL ASM MON}}
\newcommand*{\amonc}{\texttt{CREDULOUS ASM MON}}

\def\Args{\ensuremath{\textit{Args}}}
\def\attacks{\ensuremath{\leadsto}}
\def\nattacks{\ensuremath{\not\leadsto}}
\def\AF{\ensuremath{(\Args, \attacks)}}

\def\pattacks{\ensuremath{\attacks_<}}
\def\npattacks{\ensuremath{\nattacks_<}}

\def\A{\ensuremath{\textsf{A}}}
\def\B{\ensuremath{\textsf{B}}}

\def\contrary{\ensuremath{\bar{}\,\bar{}\,\bar{}\,}}
\def\abaf{\ensuremath{(\LL, \R, \A, \contrary)}} 
\def\abafp{\ensuremath{(\LL, \R, \A, \contrary, \leqslant)}} 
\def\abafe{\ensuremath{(\LL, \R, \A, \contrary, \emptyset)}}

\def\pattacks{\ensuremath{\attacks_<}}
\def\npattacks{\ensuremath{\nattacks_<}}

\def\cn{\textit{Cn}}
 
\def\Def{\textit{Def}} 
\def\ot{\ensuremath{\leftarrow}} 
\newcommand{\contr}[1]{\ensuremath\overline{#1}}

\def\asma{\ensuremath{\alpha}}
\def\asmb{\ensuremath{\beta}}

\def\asme{\ensuremath{\varepsilon}} 

\def\asmx{\ensuremath{x}} 
\def\asmp{\ensuremath{p}} 
\def\asmq{\ensuremath{q}} 
\def\asmA{\ensuremath{A}}

\def\asmE{\ensuremath{E}}

\def\A{\ensuremath{\mathcal{A}}}
\def\B{\ensuremath{\mathcal{B}}}

\def\F{\ensuremath{\mathcal{F}}}

\def\LL{\ensuremath{\mathcal{L}}}

\def\R{\ensuremath{\mathcal{R}}}

\newcommand*{\principleI}{Principle I}
\newcommand*{\maximal}{the Principle of Maximal Elements}
\newcommand*{\conflict}{the Principle of Conflict Preservation}
\newcommand*{\emptypref}{the Principle of Empty Preferences}

\setlist{nolistsep}
\nocopyright 
\begin{document}

\title{Properties of \abap~for Non-Monotonic Reasoning\thanks{This is a revised version of the paper presented at the workshop.}}
\author{Kristijonas \v Cyras \and Francesca Toni \\ Imperial College London, UK}
\maketitle

\begin{abstract}
We investigate properties of \abap, a formalism that extends the well studied structured argumentation formalism Assumption-Based Argumentation (ABA) with a preference handling mechanism. 
In particular, we establish desirable properties that \abap~semantics exhibit. 
These pave way to the satisfaction by \abap~of some (arguably) desirable principles of preference handling in argumentation and non-monotonic reasoning,
as well as non-monotonic inference properties of \abap~under various semantics. 
\end{abstract}

\section{Introduction}
\label{sec:Introduction}

Recent decades have seen a number of non-monotonic reasoning (NMR) formalisms advanced (see e.g.~\cite{Brewka:Nemiela:Truszczynski:2008-Handbook} for an overview). 
Since preferences are ubiquitous in common-sense reasoning, 
there has been a considerable effort to integrate preference information within NMR formalisms (cf.~e.g.~\cite{Brewka:Nemiela:Truszczynski:2008,Delgrande:Schaub:Tompits:Wang:2004,Domshlak:Hullermeier:Kaci:Prade:2011,Kaci:2011}). 
To evaluate distinct formalisms, various properties of both non-monotonic inference and preference handling have been proposed, 
see e.g.~\cite{Makinson:1988,Kraus:Lehmann:Magidor:1990,Brewka:Eiter:1999,Brewka:Truszczynski:Woltran:2010,Simko:2014}.

Meanwhile, argumentation (as overviewed in \cite{Rahwan:Simari:2009}) has become an established branch of AI widely used for NMR 
(see e.g.~\cite{Dung:1995,Bondarenko:Dung:Kowalski:Toni:1997,Modgil:Prakken:2013}).
Broadly speaking, information in argumentation is represented via \emph{arguments}, while \emph{attacks} among them indicate conflicts. 
Procedures, known as \emph{argumentation semantics}, are employed to select \emph{extensions}, i.e.~sets of collectively acceptable arguments. 
Preferences in argumentation also play a significant role (cf.~e.g.~\cite{Simari:Loui:1992,Kaci:2011}), 
by allowing to, for instance, discriminate among arguments or extensions. 
Over the years, numerous formalisms of argumentation with preferences have been presented (see Section \ref{sec:Related Work}) 
and some properties for argumentation with preferences indicated (e.g.~\cite{Brewka:Truszczynski:Woltran:2010,Modgil:Prakken:2013,Amgoud:Vesic:2014,Dung:2016}). 

NMR properties are also adaptable to argumentation setting. 
For example, the well known non-monotonic inference properties of \emph{Cautious Monotonicity} and \emph{Cumulative Transitivity} 
(cf.~\cite{Makinson:1988,Kraus:Lehmann:Magidor:1990}) 
concern what happens when a conclusion reached through a reasoning process is added to the knowledge base to reason with anew. 
These properties have been cast with respect to extensions in argumentation, in e.g.~\cite{Cyras:Toni:2015,Dung:2016}. 

Preference handling properties for NMR can be phrased in terms of extensions in argumentation too. 
For instance, the well known Principle I from \cite{Brewka:Eiter:1999} regarding preferred answer sets can be applied to argumentation semantics thus: 
if two extensions $E_1$ and $E_2$ coincide except for two arguments $\A \in E_1 \setminus E_2$ and $\B \in E_2 \setminus E_1$ such that $\A$ is preferred over $\B$, 
then $E_2$ should not be chosen as a `preferable' extension. 
Likewise, a common property of NMR says that, 
in the absence of preference information, 
a formalism extended with a preference handling mechanism should return the same extensions as the preference-free version of the formalism (see e.g.~\cite{Brewka:Truszczynski:Woltran:2010,Simko:2014}).

In this paper, drawing from the above mentioned works, 
we investigate various properties of a recently proposed NMR formalism \abap~\cite{Cyras:Toni:2016-KR}.
\abap~extends with a preference handling mechanism a well established argumentation formalism, 
Assumption-Based Argumentation (ABA) \cite{Bondarenko:Dung:Kowalski:Toni:1997,Toni:2014}. 
Whereas a common way to approach preferences in argumentation 
is to use preference information to \emph{discard} the attacks from arguments that are less preferred than the ones they attack 
(see e.g.~\cite{Amgoud:Cayrol:2002,Bench-Capon:2003,Kaci:Torre:2008,Brewka:Ellmauthaler:Strass:Wallner:Woltran:2013,Besnard:Garcia:Hunter:Modgil:Prakken:Simari:Toni:2014}), 
\abap~instead \emph{reverses} such attacks. 
We show that \abap's method of accounting for preferences satisfies (arguably) desirable properties.

On the one hand, we consider preference handling properties from 
\cite{Brewka:Eiter:1999,Brewka:Truszczynski:Woltran:2010,Amgoud:Vesic:2014} 
and show their satisfaction under various \abap~semantics. 
On the other hand, building on the investigations of Cumulative Transitivity and Cautious Monotonicity for ABA \cite{Cyras:Toni:2015}, 
we analyse \abap~in the light of these non-monotonic inference properties. 
In addition, we make use of the well known principle of \emph{Contraposition} of rules (see e.g.~\cite{Modgil:Prakken:2013}) 
and prove it guarantees that \abap~semantics satisfy desirable properties akin to those in e.g.~\cite{Dung:1995,Bondarenko:Dung:Kowalski:Toni:1997,Modgil:Prakken:2013}. 

The paper is organized as follows. 
Sections \ref{sec:Preliminaries} and \ref{sec:ABA+} give preliminaries on ABA and \abap. 
In Section \ref{sec:Properties} \abap~semantics are analysed. 
Preference handling properties of \abap~are studied in Section \ref{sec:Preference Handling}, 
while Section \ref{sec:NMR} concerns \abap~and non-monotonic inference properties. 
After discussing related work (Section \ref{sec:Related Work}), we conclude in Section \ref{sec:Conclusions}.

\section{Preliminaries}
\label{sec:Preliminaries}

We base the following ABA background on \cite{Toni:2014}.

\begin{definition}
\label{definition:ABA framework}
An \emph{ABA framework} is a tuple $\abaf$, where:
\begin{itemize}
\item $(\mathcal{L}, \mathcal{R})$ is a deductive system with a language $\mathcal{L}$ and a set $\mathcal{R}$ of rules of the form $\varphi_0 \leftarrow \varphi_1, \ldots, \varphi_m$ with $m \geqslant 0$ and $\varphi_i \in \mathcal{L}$ for $i \in \{ 0, \ldots, m \}$; 
$\varphi_0$ is referred to as the \emph{head} of the rule, and $\varphi_1, \ldots, \varphi_m$ is referred to as the \emph{body} of the rule;
if $m = 0$, then the rule $\varphi_0 \leftarrow \varphi_1, \ldots, \varphi_m$ is written as $\varphi_0 \leftarrow \top$ and is said to have an empty body;
\item $\mathcal{A} \subseteq \mathcal{L}$ is a non-empty set, whose elements are referred to as \emph{assumptions};
\item $\contrary : \mathcal{A} \to \mathcal{L}$ is a total map: for $\alpha \in \mathcal{A}$, the $\mathcal{L}$-formula $\overline{\alpha}$ is referred to as the \emph{contrary} of $\alpha$.
\end{itemize}
\end{definition}

We focus on \emph{flat} ABA frameworks, where no assumption is the head of any rule. 
Flat ABA frameworks are very common, and capture, as instances, widely used paradigms of non-monotonic reasoning, such as Logic Programming and Default Logic (see e.g.~\cite{Bondarenko:Dung:Kowalski:Toni:1997}). 

\begin{definition}
\label{definition:Deduction}
A \emph{deduction for $\varphi \in \mathcal{L}$ supported by $S \subseteq \mathcal{L}$ and $R \subseteq \mathcal{R}$}, 
denoted by $S \vdash^R \varphi$, 
is a finite tree with the root labelled by $\varphi$, leaves labelled by $\top$ or elements from $S$, the children of non-leaf nodes $\psi$ labelled by the elements of the body of some rule from $\mathcal{R}$ with head $\psi$, and $R$ being the set of all such rules. 
For $E \subseteq \mathcal{L}$, the \emph{conclusions} $\cn(E)$ of $E$ is the set of elements with deductions supported by $S \subseteq E$ and some $R \subseteq \mathcal{R}$, 
i.e.~$\cn(E) = \{ \varphi \in \mathcal{L}~:~\exists~S \vdash^R \varphi,~S \subseteq E,~R \subseteq \mathcal{R} \}$.
\end{definition}

Assumption-level attacks in ABA are defined thus.
\begin{definition}
\label{definition:Assumption-level attack}
A set $A \subseteq \mathcal{A}$ \emph{attacks} a set $B \subseteq \mathcal{A}$, denoted $A \attacks B$, if there is a deduction $A' \vdash^{R} \overline{\beta}$, for some $\beta \in B$, supported by some $A' \subseteq A$ and $R \subseteq \mathcal{R}$. 
For $E \subseteq \mathcal{A}$, also called an \emph{extension}, we say that:
\begin{itemize}
\item $E$ is \emph{conflict-free} if $E \nattacks E$; 
\item $E$ \emph{defends} $\alpha \in \mathcal{A}$ if for all $B \attacks \{ \alpha \}$ it holds that $E \attacks B$; 
\item $E$ is \emph{admissible} if $E$ is conflict-free and defends all $\alpha \in E$.
\end{itemize}
\end{definition}

The most standard ABA semantics are as follows.
\begin{definition}
\label{definition:ABA semantics} 
A conflict-free set $E \subseteq \mathcal{A}$ is: 
\begin{itemize}
\item \emph{stable}, if $E \attacks \{ \beta \}$ for every $\{ \beta \} \subseteq \mathcal{A} \setminus E$; 
\item \emph{complete} if $E$ is admissible and contains every assumption it defends; 
\item \emph{preferred} if $E$ is $\subseteq$-maximally admissible; 
\item \emph{grounded} if $E$ is $\subseteq$-minimally complete; 
\item \emph{ideal} if $E$ is $\subseteq$-maximal such that $E$ is admissible and contained in all preferred extensions.
\end{itemize}
\end{definition}

\begin{example}
\label{example:background}
Let $\mathcal{L} = \{ \alpha, \beta, \overline{\alpha}, \overline{\beta} \}$, 
$\mathcal{R} = \{ \overline{\alpha} \leftarrow \beta \}$ 
and $\mathcal{A} = \{ \alpha, \beta \}$. 
In $\abaf$, $\{ \beta \}$ attacks both $\{ \alpha \}$ and $\{ \alpha, \beta \}$, while $\{ \alpha, \beta \}$ attacks itself and $\{ \alpha \}$. 
$\abaf$ can be graphically represented via its \emph{assumption framework}, pictured below 
(in illustrations of assumption frameworks, nodes hold sets of assumptions while directed edges indicate attacks): 
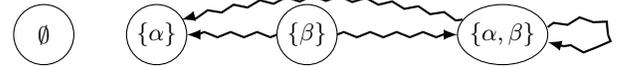
\begin{center}
\footnotesize
\begin{tikzpicture}
\node at (0.5, 0) {$\emptyset$}; 
\draw (0.5, 0) ellipse (0.4 cm and 0.4 cm);
\node at (2, 0) {$\{ \alpha \}$}; 
\draw (2, 0) ellipse (0.4 cm and 0.4 cm);
\node at (4, 0) {$\{ \beta \}$}; 
\draw (4, 0) ellipse (0.4 cm and 0.4 cm);
\node at (6.6, 0) {$\{ \alpha, \beta \}$}; 
\draw (6.6, 0) ellipse (0.6 cm and 0.4 cm);

\draw[attack] (3.6, 0) to (2.4, 0); 
\draw[attack] (4.4, 0) to (6, 0); 
\draw[self-attack, out=15, in = 90] (7.2, 0.1) to (8, 0);
\draw[attack, out=270, in = 345] (8, 0) to (7.2, -0.1); 
\draw[attack, out=165, in = 15] (6.07, 0.2) to (2.33, 0.2); 
\end{tikzpicture}
\end{center}
This $\abaf$ has a unique complete extension $\{ \beta \}$, which is also grounded, ideal, preferred and stable, and has conclusions $\cn(\{ \beta \}) = \{ \overline{\alpha}, \beta \}$.
\end{example}

\section{\abap}
\label{sec:ABA+}

\abap~\cite{Cyras:Toni:2016-KR} extends ABA with preferences as follows.

\begin{definition}
\label{definition:ABA+ framework}
An \emph{\abap~framework} is any tuple $\abafp$, where $\abaf$ is an ABA framework and $\leqslant$ is a preorder (i.e.~a transitive and reflexive binary relation) on $\mathcal{A}$.
\end{definition}

Differently from e.g.~\cite{Modgil:Prakken:2013,Modgil:Prakken:2014,Garcia:Simari:2014}, \abap~considers preferences on assumptions rather than (defeasible) rules. 
This is not, however, a conceptual difference, since assumptions are the only defeasible component in \abap. 

Unless stated differently, we consider a fixed, but otherwise arbitrary \abap~framework $\abafp$, and implicitly assume $\abaf$ to be its underlying ABA framework. 
The strict counterpart $<$ of $\leqslant$ is defined as $\alpha < \beta$ iff $\alpha \leqslant \beta$ and $\beta \nleqslant \alpha$, for any $\alpha$ and $\beta$.

\abap~attack relation is given thus. 

\begin{definition}
\label{definition:<-attack}
A set $A \subseteq \mathcal{A}$ of assumptions \emph{$<$-attacks} a set $B \subseteq \mathcal{A}$ of assumptions, written as $A \pattacks B$, if:
\begin{itemize}
\item either there is a deduction $A' \vdash^{R} \overline{\beta}$, for some $\beta \in B$, supported by $A' \subseteq A$, and $\nexists \alpha' \in A'$ with $\alpha' < \beta$;
\item or there is a deduction $B' \vdash^{R} \overline{\alpha}$, for some $\alpha \in A$, supported by $B' \subseteq B$, and $\exists \beta' \in B'$ with $\beta' < \alpha$.
\end{itemize}
The first type of attack is called \emph{normal}, and the second one \emph{reverse}.
\end{definition}
 
\abap~requires a standard ABA attack to be reversed whenever the attacker has an assumption less preferred than the one attacked. 
The following example illustrates.

\begin{example}
\label{example:reverse attack}
Recall $\abaf$ from Example \ref{example:background}. 
Suppose $\beta < \alpha$.  
In the \abap~framework $\abafp$, $\{ \beta \}$ `tries' to attack $\{ \alpha \}$, but is prevented by the preference $\beta < \alpha$. 
Instead, $\{ \alpha \}$ $<$-attacks $\{ \beta \}$, and likewise $\{ \alpha, \beta \}$, via reverse attack, 
and the latter $<$-attacks both itself and $\{ \beta \}$ via reverse attack. 
$\abafp$ can be represented graphically as follows 
(reverse attacks in assumption frameworks will be denoted by dotted arrows):
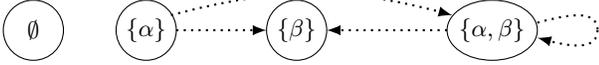
\begin{center}
\footnotesize
\begin{tikzpicture}
\node at (0.5, 0) {$\emptyset$}; 
\draw (0.5, 0) ellipse (0.4 cm and 0.4 cm);
\node at (2, 0) {$\{ \alpha \}$}; 
\draw (2, 0) ellipse (0.4 cm and 0.4 cm);
\node at (4, 0) {$\{ \beta \}$}; 
\draw (4, 0) ellipse (0.4 cm and 0.4 cm);
\node at (6.6, 0) {$\{ \alpha, \beta \}$}; 
\draw (6.6, 0) ellipse (0.6 cm and 0.4 cm);

\draw[reverse attack] (2.4, 0) to (3.6, 0); 
\draw[reverse attack] (6, 0) to (4.4, 0); 
\draw[reverse self-attack, out=15, in = 90] (7.2, 0.1) to (8, 0);
\draw[reverse attack, out=270, in = 345] (8, 0) to (7.2, -0.1); 
\draw[reverse attack, out=15, in = 165] (2.33, 0.2) to (6.07, 0.2); 
\end{tikzpicture}
\end{center}
In contrast with the ABA framework, where $\{ \beta \}$ is unattacked and generates an attack on $\{ \alpha \}$, 
in the \abap~framework, $\{ \alpha \}$ is $<$-unattacked and $<$-attacks all sets of assumptions that contain $\beta$. 
This concords with the intended meaning of the preference $\beta < \alpha$, that the conflict should be resolved in favour of $\alpha$. 
\end{example}

This concept of $<$-attack reflects the interplay between deductions, contraries and preferences, by representing inherent conflicts among sets of assumptions while accounting for preference information.  
Normal attacks follow the standard notion of attack in ABA, additionally, 
preventing the attack to succeed when the attacker uses assumptions less preferred than the one attacked. 
Reverse attacks, meanwhile, resolve the conflict between two sets of assumptions by favouring the one containing an assumption whose contrary is deduced, 
over the one which uses less preferred assumptions to deduce that contrary. 

The notions of conflict-freeness and defence w.r.t.~$\pattacks$, and \abap~semantics are given as follows.

\begin{definition}
\label{definition:<-conflict-freeness}
For $E \subseteq \mathcal{A}$ we say that:
\begin{itemize}
\item $E$ is \emph{$<$-conflict-free} if $E \npattacks E$; 
\item $E$ \emph{$<$-defends} $\alpha \in \mathcal{A}$ if for all $B \pattacks \{ \alpha \}$ it holds that $E \pattacks B$; 
\item $E$ \emph{$<$-defends} $\asmA \subseteq \mathcal{A}$ if for all $B \pattacks \asmA$ it holds that $E \pattacks B$; and 
\item $E$ is \emph{$<$-admissible} if $E$ is $<$-conflict-free and $<$-defends $E$. 
\end{itemize}
\end{definition}

In Example \ref{example:reverse attack}, $\emptyset$, $\{ \alpha \}$ and $\{ \beta \}$ are conflict-free in $\abaf$ and $<$-conflict-free in $\abafp$, whereas $\{ \alpha, \beta \}$ is not ($<$-)conflict-free in either framework.


\begin{definition}
\label{definition:ABA+ semantics} 
A $<$-conflict-free extension $E \subseteq \mathcal{A}$ is: 
\begin{itemize}
\item \emph{$<$-stable} if $E \pattacks \{ \alpha \}$ for every $\{ \alpha \} \subseteq \mathcal{A} \setminus E$;
\item \emph{$<$-complete} if $E$ is $<$-admissible and contains every assumption it $<$-defends; 
\item \emph{$<$-preferred} if $E$ is $\subseteq$-maximally $<$-admissible; 
\item \emph{$<$-grounded} if $E$ is $\subseteq$-minimally $<$-complete; 
\item \emph{$<$-ideal} if $E$ is $\subseteq$-maximal such that $E$ is $<$-admissible and contained in all $<$-preferred extensions.
\end{itemize}
\end{definition}

In Example \ref{example:reverse attack}, $\{ \alpha \}$ is a unique $<$-stable, $<$-complete, $<$-preferred, $<$-grounded and $<$-ideal extension. 

Henceforth, we assume $\sigma \in \{$\semantics$\}$ and use $<$-$\sigma$ to denote any \abap~semantics.

We recall several features that \abap~possesses and that will be used later. 

\begin{lemma}
\label{lemma:attacks on supersets}
Let $A' \subseteq A \subseteq \mathcal{A}$ and $B' \subseteq B \subseteq \mathcal{A}$ be given. If $A' \pattacks B'$, then $A \pattacks B$.
\end{lemma}

\begin{lemma}
\label{lemma:attacks}
For any $A, B \subseteq \mathcal{A}$:
\begin{itemize}
\item if $A \attacks B$, then either $A \pattacks B$ or $B \pattacks A$;
\item if $A \pattacks B$, then either $A \attacks B$ or $B \attacks A$.
\end{itemize}
\end{lemma}

\section{Properties of \abap~Semantics}
\label{sec:Properties}

To ensure that the familiar relations between semantics carry from ABA over to \abap, we want to guarantee the so-called Fundamental Lemma \cite{Dung:1995,Bondarenko:Dung:Kowalski:Toni:1997} (see below). 
To this end, we follow the well established structured argumentation formalism  \aspicp~\cite{Modgil:Prakken:2013,Modgil:Prakken:2014} and impose the principle of \emph{Contraposition},  reformulated for \abap~as follows.

\begin{axiom}
\label{axiom:Contraposition}
$\abafp$ satisfies \textbf{\contraposition} if 
for all $A \subseteq \mathcal{A}$, $R \subseteq \mathcal{R}$ and $\beta \in \mathcal{A}$ it holds that 
if $A \vdash^R \overline{\beta}$, then for every $\alpha \in A$, 
there is $R_{\alpha} \subseteq \mathcal{R}$ with $(A \setminus \{ \alpha \}) \cup \{ \beta \} \vdash^{R_{\alpha}} \overline{\alpha}$.
\end{axiom}

This axiom requires that if an assumption plays a role in deriving the contrary of another assumption, then it should contrapositively be possible for the latter to induce a derivation of the contrary of the former assumption too. 
The following example illustrates the effect Contraposition has in \abap.

\begin{example}
\label{example:fundamental}
Let $\mathcal{R} = \{ \overline{\beta} \leftarrow \alpha, \gamma \}$, 
$\mathcal{A} = \{ \alpha, \beta, \gamma \}$ and $\alpha < \beta$, $\alpha < \gamma$. 
(The language and the contrary mapping are implicit from $\mathcal{R}$ and $\mathcal{A}$.) 
This \abap~framework $\abafp$ does not satisfy \contraposition. 
Its assumption framework (omitting $\emptyset$, $\mathcal{A}$ and $<$-attacks to and from $\mathcal{A}$) is shown below:
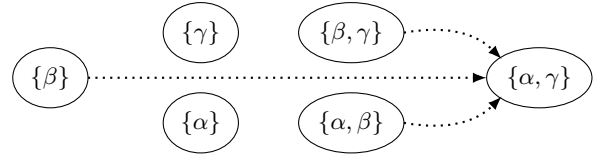
\begin{center}
\footnotesize
\begin{tikzpicture}
\node at (3, 0) {$\{ \alpha \}$}; 
\draw (3, 0) ellipse (0.5 cm and 0.4 cm);
\node at (1, 0.6) {$\{ \beta \}$}; 
\draw (1, 0.6) ellipse (0.5 cm and 0.4 cm);
\node at (3, 1.2) {$\{ \gamma \}$}; 
\draw (3, 1.2) ellipse (0.5 cm and 0.4 cm);
\node at (5, 0) {$\{ \alpha, \beta \}$}; 
\draw (5, 0) ellipse (0.7 cm and 0.4 cm); 
\node at (7.5, 0.6) {$\{ \alpha, \gamma \}$}; 
\draw (7.5, 0.6) ellipse (0.7 cm and 0.4 cm); 
\node at (5, 1.2) {$\{ \beta, \gamma \}$}; 
\draw (5, 1.2) ellipse (0.7 cm and 0.4 cm); 

\draw [reverse attack] (1.5, 0.6) to (6.8, 0.6); 
\draw [reverse attack, out=10, in=135] (5.7, 1.2) to (7, 0.85); 
\draw [reverse attack, out=350, in=225] (5.7, 0) to (7, 0.35); 
\end{tikzpicture}
\end{center}
There are no extensions under, for instance, $<$-complete semantics, 
because all the singletons $\{ \alpha \}$, $\{ \beta \}$ and $\{ \gamma \}$ are $<$-unattacked, 
but $\{ \alpha, \beta, \gamma \}$ is not $<$-conflict-free. 

If the rules $\overline{\alpha} \leftarrow \beta, \gamma$ and $\overline{\gamma} \leftarrow \alpha, \beta$ are added to $\mathcal{R}$ to constitute $\mathcal{R}'$, 
then the resulting $(\mathcal{L}, \mathcal{R}', \mathcal{A}, \contrary, \leqslant)$ 
satisfies \contraposition~and its assumption framework looks as follows
($<$-attacks that are both normal and reverse are depicted as solid directed edges):
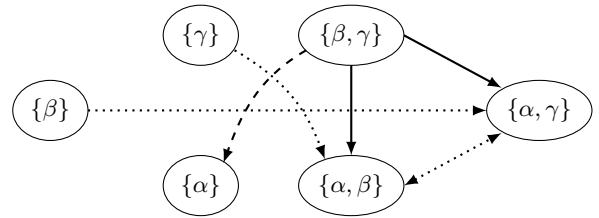
\begin{center}
\footnotesize
\begin{tikzpicture}
\node at (3, 0) {$\{ \alpha \}$}; 
\draw (3, 0) ellipse (0.5 cm and 0.4 cm);
\node at (1, 1) {$\{ \beta \}$}; 
\draw (1, 1) ellipse (0.5 cm and 0.4 cm);
\node at (3, 2) {$\{ \gamma \}$}; 
\draw (3, 2) ellipse (0.5 cm and 0.4 cm);
\node at (5, 0) {$\{ \alpha, \beta \}$}; 
\draw (5, 0) ellipse (0.7 cm and 0.4 cm); 
\node at (7.5, 1) {$\{ \alpha, \gamma \}$}; 
\draw (7.5, 1) ellipse (0.7 cm and 0.4 cm); 
\node at (5, 2) {$\{ \beta, \gamma \}$}; 
\draw (5, 2) ellipse (0.7 cm and 0.4 cm); 

\draw [reverse attack] (1.5, 1) to (6.8, 1); 
\draw [pattack] (5.7, 2) to (7, 1.3); 
\draw [symmetric reverse attack] (5.7, 0) to (7, 0.7); 
\draw [reverse attack, out=330, in=110] (3.45, 1.8) to (4.65, 0.35); 
\draw [pattack] (5, 1.6) to (5, 0.4); 
\draw [normal attack, out=210, in=70] (4.4, 1.8) to (3.3, 0.3); 
\end{tikzpicture}
\end{center}
Here, $\{ \beta, \gamma \}$ is a unique $<$-complete extension. 
\end{example}

We prove next that in the presence of Contraposition, the Fundamental Lemma is guaranteed to hold in \abap.

\begin{lemma}
\label{lemma:Fundamental}
Suppose that $\abafp$ satisfies \contraposition. 
Let $S \subseteq \mathcal{A}$ be $<$-admissible and assume that $S$ $<$-defends $\alpha, \alpha' \in \mathcal{A}$. 
Then $S \cup \{ \alpha \}$ is $<$-admissible and $<$-defends $\alpha'$.
\end{lemma}

\begin{proof}
Note that if $\alpha \in S$, then $S \cup \{ \alpha \}$ is trivially $<$-admissible. 
So assume $\alpha \not\in S$ and suppose for a contradiction that $S \cup \{ \alpha \}$ is not $<$-admissible. Then it is either not $<$-conflict-free, or does not $<$-defend itself. 
Suppose first $S \cup \{ \alpha \} \pattacks S \cup \{ \alpha \}$ via either (1) normal or (2) reverse attack. 
We show that either leads to a contradiction. 

1.~$S \cup \{ \alpha \} \pattacks S \cup \{ \alpha \}$ via normal attack. 
As $S$ is $<$-conflict-free and $<$-defends $\alpha$, this $<$-attack must involve $\alpha$. 
I.e.~$S' \cup \{ \alpha \} \vdash^{R} \overline{\beta}$ for some $S' \subseteq S$ and $\beta \in S \cup \{ \alpha \}$, and $\forall s' \in S' \cup \{ \alpha \}$ we find $s' \not< \beta$. 
If $\beta = \alpha$, then $S' \cup \{ \alpha \} \pattacks \{ \alpha \}$, and so $S \pattacks S' \cup \{ \alpha \}$. 
Else, if $\beta \in S'$, then $S' \cup \{ \alpha \} \pattacks S$, and so $S \pattacks S' \cup \{ \alpha \}$ as well. 
We show that we can similarly obtain $S \pattacks S' \cup \{ \alpha \}$ in case (2) too. 

2.~$S \cup \{ \alpha \} \pattacks S \cup \{ \alpha \}$ via reverse attack. 
As in 1., this $<$-attack must involve $\alpha$, i.e. 
$S' \cup \{ \alpha \} \vdash^R \overline{\beta}$ for some $S' \subseteq S$ and $\beta \in S \cup \{ \alpha \}$, and $\exists s' \in S' \cup \{ \alpha \}$ such that $s' < \beta$. 
If $\beta \in S$, then $S \pattacks S' \cup \{ \alpha \}$. 
Else, if $\beta = \alpha$, then $s' \neq \alpha$ (by asymmetry of $<$), and using \contraposition~we find 
$A \vdash^{R'} \overline{s'}$ for $A \subseteq (S' \cup \{ \alpha \}) \setminus \{ s' \}$, so that $S' \cup \{ \alpha \} \attacks S$. 
Then, by Lemma \ref{lemma:attacks}, 
either $S' \cup \{ \alpha \} \pattacks S$ or $S \pattacks S' \cup \{ \alpha \}$, which yields $S \pattacks S' \cup \{ \alpha \}$ in any case. 

In either (1) or (2), $S \pattacks S' \cup \{ \alpha \}$, and as $S$ is $<$-conflict-free and $<$-defends $\alpha$, this $<$-attack must be reverse and involve $\alpha$: 
$A_1 \cup \{ \alpha \} \vdash^{R_1} \overline{s_1}$, $s_1 \in S$, $A_1 \subseteq S'$, and $\exists s'_1 \in A_1$ with $s'_1 < s_1$. 
Without loss of generality take $s'_1$ to be $\leqslant$-minimal such. 
By \contraposition, there is $S_1 \cup \{ \alpha \} \vdash^{R'_1} \overline{s'_1}$ with $S_1 \subseteq (A_1 \setminus \{ s'_1 \}) \cup \{ s_1 \}$ and $\forall x \in S_1~~x \not< s'_1$ (by $\leqslant$-minimality of $s'_1$). 
That is, $S_1 \cup \{ \alpha \} \pattacks A_1$, so we find $S \pattacks S_1 \cup \{ \alpha \}$, again via reverse attack involving $\alpha$: 
$A_2 \cup \{ \alpha \} \vdash^{R_2} \overline{s_2}$, $s_2 \in S$, $A_2 \subseteq S_1$, and $\exists s'_2 \in A_2$ with $s'_2 < s_2$. 
We again impose $\leqslant$-minimality on $s'_2$ and by \contraposition~get
$S_2 \cup \{ \alpha \} \vdash^{R'_2} \overline{s'_2}$, $S_2 \subseteq (A_2 \setminus \{ s'_2 \}) \cup \{ s_2 \}$ and $\forall x \in S_2~~x \not< s'_2$.

As deductions are finite and $<$ asymmetric, the procedure described above will eventually exhaust pairs of $s'_k \in A_k$ and $s_k \in S_k$ such that $s'_k < s_k$, so that $S \pattacks S_k \cup \{ \alpha \}$ will have to be a normal attack, for some $k$. 
But this leads to a contradiction to $S$ being $<$-admissible and $<$-defending $\alpha$.

Hence, by contradiction, $S \cup \{ \alpha \}$ is $<$-conflict-free. 

We now want to show that $S \cup \{ \alpha \}$ $<$-defends itself. So let $B \pattacks S \cup \{ \alpha \}$. 
As $S$ is $<$-admissible and $<$-defends $\alpha$, we consider this $<$-attack to be reverse and involving $\alpha$: 
$S' \cup \{ \alpha \} \vdash^R \overline{\beta_1}$, $S' \subseteq S$, $\beta_1 \in B$, and $\exists s' \in S' \cup \{ \alpha \}$ with $s' < \beta_1$. 
By \contraposition, $S_1 \vdash^{R'_1} \overline{s'}$, $S_1 \subseteq ((S' \cup \{ \alpha \}) \setminus \{ s' \}) \cup \{ \beta_1 \}$. 
Thus, $S_1 \attacks \{ s' \}$, whence $S \cup \{ \alpha \} \pattacks S_1$. 
This $<$-attack cannot be normal on $(S' \cup \{ \alpha \}) \setminus \{ s' \}$, due to $<$-conflict-freeness of $S \cup \{ \alpha \}$; while, if it is normal on $\beta_1$, then $S \cup \{ \alpha \} \pattacks B$, as required. 
Else, $S \cup \{ \alpha \} \pattacks S_1$ via reverse attack: 
$B_1 \vdash^{R_1} \overline{s_1}$, $s_1 \in S \cup \{ \alpha \}$, $B_1 \subseteq S_1$, and $\exists s'_1 \in B_1$ with $s'_1 < s_1$. 
Due to $<$-conflict-freeness of $S \cup \{ \alpha \}$, we find $\beta_1 \in B_1$. Then again, by \contraposition, we find 
$S_2 \vdash^{R'_2} \overline{s'_1}$, $S_2 \subseteq (B_1 \setminus \{ s'_1 \}) \cup \{ s_1 \}$, and $\beta_1 \in S_2$. 
Like with the proof of $<$-conflict-freeness, this process must terminate with a normal attack $S \cup \{ \alpha \} \pattacks B$, so that $S \cup \{ \alpha \}$ eventually $<$-defends itself. 

Finally, given that $S$ $<$-defends $\alpha'$ to begin with, using Lemma \ref{lemma:attacks on supersets} we conclude that $S \cup \{ \alpha \}$ $<$-defends $\alpha'$ too.
\end{proof}

For the rest of this section, we assume that $\abafp$ satisfies \contraposition. 

We can now define the $<$-defence operator $\Def$, inspired by \cite{Dung:1995}. 

\begin{definition}
\label{definition:<-Defence operator}
$\Def: \wp(\mathcal{A}) \to \wp(\mathcal{A})$ is defined as follows: 
for $A \subseteq \mathcal{A}$, $\Def(A) = \{ \alpha \in \mathcal{A}~:~A \text{ $<$-defends } \alpha \}$. 
\end{definition}

By Lemma \ref{lemma:attacks on supersets}, $\Def$ is monotonic: if $A \subseteq B \subseteq \mathcal{A}$, then $\Def(A) \subseteq \Def(B)$. 
Hence, $\Def$ has a unique least fixed point, which is in addition a unique $<$-grounded extension of $\abafp$, as shown next. 

\begin{proposition}
\label{proposition:unique <-grounded}
$\abafp$ admits a unique $<$-grounded extension.
\end{proposition}

\begin{proof}
First, observe that $\emptyset$ is $<$-admissible in $\abafp$. 
The least fixed point $G$ can be given as $\bigcup_{i \in \mathbb{N}} \Def^{~i}(\emptyset)$. 
By Lemma \ref{lemma:Fundamental}, $G$ is $<$-admissible. 
It is clearly $<$-complete (as $G = \Def(G)$) and unique $\subseteq$-minimal such (as the least fixed point). 
Hence, $G$ is a unique $<$-grounded extension of $\abafp$.
\end{proof}

As a consequence of Proposition \ref{proposition:unique <-grounded}, we get the following.

\begin{corollary}
\label{corollary:<-complete exists}
$\abafp$ admits a $<$-complete extension.
\end{corollary}

Using Lemma \ref{lemma:Fundamental}, we can prove the following results. 

\begin{proposition}
\label{proposition:<-preferred always exist}
$\abafp$ admits a $<$-preferred extension. 
\end{proposition}

\begin{proof}
By Lemma \ref{lemma:Fundamental}, the collection of $<$-admissible supersets of $\emptyset$ is partially ordered by subset inclusion $\subseteq$, so any sequence $\emptyset \subseteq A_1 \subseteq \ldots \subseteq A_n \subseteq \ldots$ of $<$-admissible sets of assumptions (for $n$ an ordinal) has an upper bound $A = \bigcup_{i \geqslant 0} A_i$. 
Then $A \subseteq \mathcal{A}$ is $<$-admissible: if it were not $<$-conflict-free, then some $A_n$ would not be either; and for any $B \pattacks A$ we have $B \pattacks A_n$, for some $n$, so that $A_n \pattacks B$, and hence $A \pattacks B$ too. 
Since every chain $\emptyset \subseteq A_1 \subseteq \ldots \subseteq A_n \subseteq \ldots$ admits an $<$-admissible upper bound, every such chain has a $\subseteq$-maximally $<$-admissible set of assumptions, according to Zorn's Lemma. 
As $\emptyset$ is $<$-admissible, $\abafp$ admits at least one $\subseteq$-maximally $<$-admissible---i.e.~a $<$-preferred---extension.
\end{proof}

\begin{proposition}
\label{proposition:<-preferred is <-complete}
Every $<$-preferred extension of $\abafp$ is a $<$-complete extension too.
\end{proposition}

\begin{proof}
Let $E$ be a $<$-preferred extension of $\abafp$ and suppose for a contradiction that it is not $<$-complete. Let $E$ $<$-defend some $\alpha \in \mathcal{A} \setminus E$. 
As $E$ is $<$-admissible, $E \cup \{ \alpha \}$ is $<$-admissible, by Lemma \ref{lemma:Fundamental}. 
But then $E$ is not $\subseteq$-maximally $<$-admissible, contrary to $E$ being $<$-preferred. 
Hence, by contradiction, $E$ must be $<$-complete.
\end{proof}

Further, as in ABA, $<$-stable semantics is subsumed by both $<$-preferred and $<$-complete semantics, as shown next.

\begin{proposition}
\label{proposition:<-stable is <-preferred}
Any $<$-stable extension of $\abafp$ is a $<$-preferred extension too.
\end{proposition}

\begin{proof}
Let $E$ be a $<$-stable extension of $\abafp$. 
As $E$ $<$-attacks every $\{ \beta \} \nsubseteq E$, it must be $\subseteq$-maximally $<$-admissible. 
Hence, $E$ is $<$-preferred.
\end{proof}

\begin{proposition}
\label{proposition:<-stable is <-complete}
Any $<$-stable extension of $\abafp$ is a $<$-complete extension too.
\end{proposition}

\begin{proof}
Let $E$ be a $<$-stable extension of $\abafp$. 
For any $\beta \not\in E$, $<$-stability of $E$ means that $E \pattacks \{ \beta \}$, and if $E$ $<$-defended $\beta$ as well, it would mean that $E \pattacks E$, contradicting its $<$-conflict-freeness. 
Hence, $E$ contains every assumption it $<$-defends, and so is $<$-complete.
\end{proof}

Finally, we consider $<$-ideal semantics.

\begin{proposition}
\label{proposition:unique <-ideal}
$\abafp$ admits a unique $<$-ideal extension.
\end{proposition}

\begin{proof}
From Proposition \ref{proposition:<-preferred always exist} we know that $\abafp$ admits $<$-preferred extensions, so let $S$ be their intersection. 
If $S = \emptyset$, then it is $<$-admissible, and so an $<$-ideal extension (unique). 
If $S \neq \emptyset$ is $<$-admissible, then it is an $<$-ideal extension (unique as well). 
Else, assume $S \neq \emptyset$ is not $<$-admissible. 
Then its $\subseteq$-maximally $<$-admissible subsets $I \subsetneq S$ are $<$-ideal extensions of $\abafp$. 
Suppose $I$ and $I'$ are two distinct $<$-admissible subsets of $S$. 
Then their union $I \cup I'$ is a subset of $S$ too, and so $<$-conflict-free. 
By Lemma \ref{lemma:Fundamental}, $I \cup I'$ $<$-defends its assumptions, so must be $<$-admissible. 
Consequently, there can be only one $\subseteq$-maximally $<$-admissible subset of $S$, i.e.~$\abafp$ has a unique $<$-ideal extension. 
\end{proof}

\begin{proposition}
\label{proposition:<-ideal is <-complete}
Any $<$-ideal extension of $\abafp$ is a $<$-complete extension too.
\end{proposition}

\begin{proof}
By Proposition \ref{proposition:unique <-ideal}, it has a unique $<$-ideal extension $I$. Suppose for a contradiction that $I$ is not $<$-complete. 
Then there is $\alpha \in \mathcal{A} \setminus I$ $<$-defended by $I$. 
Such $\alpha$ must be contained in the intersection $S$ of $<$-preferred extensions of $\abafp$, 
because $I \subseteq S$ $<$-defends $\alpha$ and every $<$-preferred extension $E$ of $\abafp$ is $<$-complete (by Proposition \ref{proposition:<-preferred is <-complete}). 
But then, $I \cup \{ \alpha \}$ is $<$-admissible, according to Lemma \ref{lemma:Fundamental}, so that $I$ is not $<$-ideal---a contradiction. 
Therefore, $I$ must be $<$-complete.
\end{proof}

These properties that \abap~exhibits in the presence of Contraposition will be used to show, in the coming sections, that \abap~satisfies certain principles of preference handling and non-monotonic reasoning.

\section{Preference Handling Properties}
\label{sec:Preference Handling}

Referring to \cite{Amgoud:Vesic:2009}, 
in \cite{Brewka:Truszczynski:Woltran:2010} the authors hinted at two (arguably) desirable properties of argumentation formalisms dealing with preferences, that concern conflict preservation and the absence of preferences. 
In the next two subsections we indicate that \abap~satisfies those properties, 
and in the following subsections show that other (arguably) desirable properties of preference handling are too satisfied by \abap.

\subsection{Conflict Preservation}
\label{subsec:Conflict Preservation}

The first property insists that extensions returned after accounting for preferences should be conflict-free with respect to attack relation not taking into account preferences. 
We formulate it as a principle applicable to \abap~as follows.

\begin{principle}
\label{principle:Conflict Preservation}
 $\abafp$ fulfils \textbf{\conflict} for $<$-$\sigma$ semantics if for all $<$-$\sigma$ extensions $E \subseteq \mathcal{A}$ of $\abafp$, for any $\alpha, \beta \in \mathcal{A}$, 
$\{ \alpha \} \attacks \{ \beta \}$ implies that either $\alpha \not\in E$ or $\beta \not\in E$. 
\end{principle}

In \cite{Cyras:Toni:2016-KR} it was shown that Lemma \ref{lemma:attacks} guarantees the following result. 

\begin{proposition}
\label{proposition:<-conflict-free iff conflict-free}
$E \subseteq \mathcal{A}$ is conflict-free in $\abaf$ iff $E$ is $<$-conflict-free in $\abafp$.
\end{proposition}

Consequently, \abap~ensures conflict preservation:

\begin{proposition}
\label{proposition:Conflict Preservation}
$\abafp$ fulfils \conflict~for any semantics $<$-$\sigma$.
\end{proposition}

\begin{proof}
Let $E$ be a $<$-$\sigma$ extension of $\abafp$, 
If $\alpha, \beta \in E$ and $\{ \alpha \} \attacks \{ \beta \}$, 
then $\{ \alpha, \beta \}$ is not conflict-free, and hence not $<$-conflict-free, by Proposition \ref{proposition:<-conflict-free iff conflict-free}. 
But then $E$ is not $<$-conflict-free either, which is a contradiction. 
Thus, either one of $\alpha$ and $\beta$ does not belong to $E$.
\end{proof}

\subsection{Empty Preferences}
\label{subsec:Empty Preferences}

The second property insists that if there are no preferences, then the extensions returned using a preference handling mechanism should be the same as those obtained without accounting for preferences. 
We formulate it as a principle applicable to \abap~as follows.

\begin{principle}
\label{principle:Empty Preferences}
Suppose that the preference relation $\leqslant$ in $\abafp$ is the strict empty ordering $\emptyset$. 
Then $\abafe$ fulfils \textbf{\emptypref} for $\emptyset$-$\sigma$ semantics if for all $\emptyset$-$\sigma$ extensions $E \subseteq \mathcal{A}$ of $\abafe$, 
$E$ is a $\sigma$ extension of $\abaf$.
\end{principle}

In \cite{Cyras:Toni:2016-KR} the following result was shown to hold.

\begin{theorem}
\label{theorem:ABA+ extends ABA}
$E \subseteq \mathcal{A}$ is a $\sigma$-extension of $\abaf$ iff $E$ is an $\emptyset$-$\sigma$ extension of $\abafe$.
\end{theorem}

This theorem, in addition to saying that \abap~is a conservative extension of ABA, immediately yields the satisfaction of the principle in question: 

\begin{proposition}
\label{proposition:Empty Preferences}
$\abafe$ fulfils \emptypref~for any semantics $\emptyset$-$\sigma$.
\end{proposition}

\subsection{Maximal Elements}
\label{subsec:Maximal Elements}

\cite{Amgoud:Vesic:2014} proposed a property concerning inclusion in extensions of the `strongest' arguments, i.e.~arguments that are maximal w.r.t.~preference ordering. We next reformulate the property to be applicable to \abap.

\begin{principle}
\label{principle:Maximal Elements}
Suppose the preference ordering $\leqslant$ of $\abafp$ is total
and further assume that the set $M = \{ \alpha \in \mathcal{A} : \nexists \beta \in \mathcal{A}$ with $\alpha < \beta \}$ is $<$-conflict-free.
 $\abafp$ fulfils \textbf{\maximal} for $<$-$\sigma$ semantics if for all $<$-$\sigma$ extensions $E \subseteq \mathcal{A}$ of $\abafp$, it holds that $M \subseteq E$.
\end{principle}

As an illustration, in Example \ref{example:reverse attack}, $\alpha$ is a unique $\leqslant$-maximal element in $\mathcal{A}$, 
and $\{ \alpha \}$ is a unique $<$-$\sigma$ extension of $\abafp$, 
whence $\abafp$ fulfils \maximal~for any semantics $<$-$\sigma$. 

Our next result shows that in general, \abap~satisfies this principle under $<$-stable and $<$-complete semantics. 

\begin{proposition}
\label{proposition:Maximal Elements}
$\abafp$ fulfils \maximal~for $<$-stable and $<$-complete semantics.
\end{proposition}

\begin{proof}
Let the preference ordering $\leqslant$ of $\abafp$ be total
and suppose $M = \{ \alpha \in \mathcal{A} : \nexists \beta \in \mathcal{A}$ with $\alpha < \beta \}$ is $<$-conflict-free. We first show that $M$ is not $<$-attacked.

Fix $\alpha \in M$ and suppose for a contradiction that for some $S \subseteq \mathcal{A}$ it holds that $S \pattacks \{ \alpha \}$. 
So either (i) $\exists B \vdash^R \overline{\alpha}$ with $B \subseteq S$ and $\forall \beta \in B~~\alpha \leqslant \beta$ or $\beta \nleqslant \alpha$, 
or (ii) $\{ \alpha \} \vdash^R \overline{\beta}$ for some $\beta \in S$ with $\alpha < \beta$. 
Note that the case (ii) cannot happen, because $\alpha$ is $\leqslant$-maximal. So consider case (i).
Since $\leqslant$ is total, it follows that $\alpha \leqslant \beta~~\forall \beta \in B$. 
But as $\alpha$ is $\leqslant$-maximal, it must also hold that $\beta \leqslant \alpha$, for any $\beta \in B$. 
From here, we show $B \subseteq M$. Indeed, fix $\beta \in B$ and assume for a contradiction that $\beta \not\in M$. 
Then $\exists \gamma \in \mathcal{A}$ such that $\beta < \gamma$. By transitivity, $\alpha < \gamma$, contradicting $\alpha$'s $\leqslant$-maximality. 
So we must have $\beta \in M$, and consequently, $B \subseteq M$. 

But now, since $\alpha \in M,~B \subseteq M$ and $B \pattacks \{ \alpha \}$, this contradicts $<$-conflict-freeness of $M$. 
Therefore, by contradiction, $S \npattacks \{ \alpha \}$, for any $S \subseteq \mathcal{A}$. 
Since $\alpha \in M$ was arbitrary, we have $M$ $<$-unattacked, as required.

If $\abafp$ admits no $<$-stable or $<$-complete extensions, then the principle is fulfilled trivially. Otherwise, let $E \subseteq \mathcal{A}$ be $<$-stable in $\abafp$. 
Pick $\alpha \in M$ and suppose for a contradiction that $\alpha \not\in E$. Then $E \pattacks \{ \alpha \}$, which is a contradiction.
Thus, $\alpha \in S$, and hence $M \subseteq S$.

Now let $E$ be a $<$-complete extension of $\abafp$ and suppose for a contradiction $M \nsubseteq E$. Then $E$ does not $<$-defend some $\alpha \in M$. 
This means that $S \pattacks M$ for some $S \subseteq \mathcal{A}$, which is a contradiction. Hence, $M \subseteq E$.
\end{proof}

This principle may, however, be violated under, say, $<$-preferred semantics: 
in Example \ref{example:fundamental}, the framework $\abafp$ to begin with,
admits $\{ \alpha, \beta \}$ as a $<$-preferred extension, 
while $\gamma \not\in \{ \alpha, \beta \}$ is a $\leqslant$-maximal element. 
However, assuming Contraposition, \maximal~is satisfied under the remaining semantics too.

\begin{corollary}
\label{corollary:Maximal Elements for <-preferred}
If $\abafp$ satisfies \contraposition, then it fulfils \maximal~for $<$-preferred/$<$-ideal/$<$-grounded semantics.
\end{corollary}

\begin{proof}
Follows from Propositions \ref{proposition:unique <-grounded}, \ref{proposition:<-preferred is <-complete}, \ref{proposition:<-ideal is <-complete} and \ref{proposition:Maximal Elements}.
\end{proof}

\subsection{Principle I}
\label{subsec:Principle I}

\cite{Brewka:Eiter:2000} formulated a principle for sound extension-based default reasoning with preferences, which we reformulate for \abap~next.

\begin{principle}
\label{principle:Preferences}
$\abafp$ fulfils \textbf{\principleI} for $<$-$\sigma$ semantics 
if for all $E, E' \subseteq \mathcal{A}$ such that $E = E_0 \cup \{ \alpha \}$ and $E' = E_0 \cup \{ \alpha' \}$ for some $E_0 \subseteq \mathcal{A}$, 
with $\alpha, \alpha' \not\in E_0$ and $\alpha' < \alpha$, 
it holds that if $E$ is a $<$-$\sigma$ extension of $\abafp$, then $E'$ is not a $<$-$\sigma$ extension of $\abafp$.
\end{principle}

This principle insists that if two coherent viewpoints of a situation differ only in that each of them contains a single assumption not contained in the other, then the viewpoint with the more preferred assumption should be chosen. 
\abap~satisfies this principle under $<$-stable semantics.

\begin{proposition}
\label{proposition:<-stable satisfies Principle I}
$\abafp$ fulfils \principleI~for $<$-stable semantics.
\end{proposition}

\begin{proof}
Suppose for a contradiction that both $E = E_0 \cup \{ \alpha \}$ and $E' = E_0 \cup \{ \alpha' \}$, where $\alpha' < \alpha$, are $<$-stable extensions of $\abafp$. 
As $E'$ is $<$-stable and $\alpha \not\in E'$, we get $E' \pattacks \{ \alpha \}$. 
As $E$ is $<$-conflict-free, we find $E_0 \npattacks \{ \alpha \}$, so (from $E' \pattacks \{ \alpha \}$ we get that):
(i) either there is $E'' \cup \{ \alpha' \} \vdash^R \overline{\alpha}$ with $E'' \subseteq E_0$ and $\varepsilon \not< \alpha~~\forall \varepsilon \in E'' \cup \{ \alpha' \}$;
(ii) or $\{ \alpha \} \vdash^R \overline{\alpha'}$ is such that $\alpha < \alpha'$. 
As $\alpha' < \alpha$, both cases lead to a contradiction, so that $E'$ is not a $<$-stable extension, provided $E$ is.
\end{proof}

In Example \ref{example:reverse attack}, $E = \{ \alpha \}$ is a unique $<$-stable extension of $\abafp$, which illustrates the principle as follows: 
take $E_0 = \emptyset$ so that $E = \{ \alpha \}$ and $E' = \{ \beta \}$, where $\beta < \alpha$. 
It is important that \principleI~is satisfied under $<$-stable semantics, because \cite{Brewka:Eiter:1999} investigated (preferred) answer sets of logic programs, 
and answer sets in Logic Programming correspond to stable extensions in ABA \cite{Bondarenko:Dung:Kowalski:Toni:1997}. 
Satisfaction of the principle gives hope that preferred answer set semantics can be captured in \abap, as answer set semantics is captured in ABA. 

\principleI, however, may be violated under $<$-preferred semantics: 
in Example \ref{example:fundamental}, $\abafp$ has two $<$-preferred extensions 
$\{ \alpha, \beta \}$ and $\{ \beta, \gamma \}$, 
and yet $\alpha < \gamma$. 
Note, though, that $(\mathcal{L}, \mathcal{R}', \mathcal{A}, \contrary, \leqslant)$ 
satisfies \contraposition~and has a unique $<$-$\sigma$ extension $\{ \beta, \gamma \}$, 
and thus fulfils \principleI~for any semantics $<$-$\sigma$. 
Based on our investigations, we conjecture that assuming Contraposition, \abap~frameworks fulfil the principle for the remaining semantics as well. 
Verifying this is left as future work.

\section{Non-Monotonic Reasoning Properties}
\label{sec:NMR}

\cite{Cyras:Toni:2015} proposed and studied the well known non-monotonic inference properties of \emph{Cautious Monotonicity} (\mon~henceforth) and \emph{Cumulative Transitivity} (\cut~henceforth) for ABA. 
Here, we investigate some of those properties for \abap. 
We first recall (some of) the properties considered and results obtained.\footnote{In \cite{Cyras:Toni:2015}, instead of sceptical/credulous (see below) the words strong/weak were used, respectively; 
we have altered the names to adhere to the more common terminology.}

Assume as given a fixed, but otherwise arbitrary (flat) ABA framework $\mathcal{F} = \abaf$. 
Let $E$ be a $\sigma$ extension of $\mathcal{F}$. 
In what follows, $E'$ will denote a $\sigma$ extension of a newly constructed ABA framework $\mathcal{F}'$. 
To avoid trivialities, we consider cases only where each of $\mathcal{F}$ and $\mathcal{F}'$ has at least one $\sigma$ extension---$E$ and $E'$ respectively.

We first recall the \texttt{STRICT} setting regarding \emph{strengthening of information}. 
Given $\psi \in \cn(E) \setminus \mathcal{A}$, 
define $\mathcal{F}' = (\mathcal{L}, \mathcal{R} \cup \{ \psi \leftarrow \top \}, \mathcal{A}, \contrary)$. 
There are four properties:

\begin{align*}
&\scuts:\\
&\text{For all extensions } E' \text{ of } \mathcal{F}' \text{ we have } \cn(E') \subseteq \cn(E); \\
&\scutc: \\
&\text{There is an extension } E' \text{ of } \mathcal{F}' \text{ with } \cn(E') \subseteq \cn(E); \\
&\smons: \\
&\text{For all extensions } E' \text{ of } \mathcal{F}' \text{ we have } \cn(E) \subseteq \cn(E'); \\
&\smonc: \\
&\text{There is an extension } E' \text{ of } \mathcal{F}' \text{ with } \cn(E) \subseteq \cn(E').
\end{align*}

Table 1 summarizes results pertaining to ABA (sceptical and credulous versions coincide under grounded and ideal semantics, 
and for other semantics the status of the credulous property is indicated in parentheses).

\begin{table}[h]
\begin{center}
\begin{tabular}{| >{\centering}m{1.8cm}<{\centering} | c | c | c | c | c | c |}
\hline
Property & Grd. & Ideal & Stable & Pref. & Cpl. \\ \hline
\scut & $\checkmark$ & $\checkmark$ & \sffamily{X} ($\checkmark$) & \sffamily{X} ($\checkmark$) & \sffamily{X} ($\checkmark$) \\ \hline

\smon & $\checkmark$ & \sffamily{X} & \sffamily{X} ($\checkmark$) & \sffamily{X} ($\checkmark$) & \sffamily{X} ($\checkmark$) \\ \hline
\end{tabular}
\end{center} 
\caption{\texttt{STRICT} \cut\,/\,\mon~for standard ABA}
\end{table}

We now recall the \texttt{ASM} setting, where conclusions that are themselves assumptions are being \emph{confirmed}. 
Given $\psi \in \cn(E) \cap \mathcal{A}$, define $\mathcal{F}' = (\mathcal{L}, \mathcal{R} \cup \{ \psi \leftarrow \top \}, \mathcal{A} \setminus \{ \psi \}, \contrary)$.\footnote{For brevity reasons, the same symbol $\contrary$ is used for both contrary mappings, and in the new framework $\mathcal{F}'$, the contrary mapping $\contrary$ is implicitly restricted to a diminished set of assumptions.}
The properties are as follows:
\begin{align*}
&\acuts: \\
&\text{For all extensions } E' \text{ of } \mathcal{F}' \text{ we have } \cn(E') \subseteq \cn(E); \\
&\acutc: \\
&\text{There is an extension } E' \text{ of } \mathcal{F}' \text{ with } \cn(E') \subseteq \cn(E); \\
&\amons: \\
&\text{For all extensions } E' \text{ of } \mathcal{F}' \text{ we have } \cn(E) \subseteq \cn(E'); \\
&\amonc: \\
&\text{There is an extension } E' \text{ of } \mathcal{F}' \text{ with } \cn(E) \subseteq \cn(E').
\end{align*}

Table 2 summarizes results regarding ABA in the \texttt{ASM} setting (notation as before).

\begin{table}[h]
\begin{center}
\begin{tabular}{| >{\centering}m{1.8cm}<{\centering} | c | c | c | c | c | c |}
\hline
Property & Grd. & Ideal & Stable & Pref. & Cpl. \\ \hline
\acut & $\checkmark$ & $\checkmark$ & \sffamily{X} ($\checkmark$) & \sffamily{X} ($\checkmark$) & \sffamily{X} ($\checkmark$) \\ \hline

\amon & $\checkmark$ & \sffamily{X} & \sffamily{X} ($\checkmark$) & \sffamily{X} ($\checkmark$) & \sffamily{X} ($\checkmark$) \\ \hline
\end{tabular}
\end{center}
\caption{\texttt{ASM} \cut\,/\,\mon~for standard ABA}
\end{table}

The non-monotonic inference properties \cut~and \mon~can be readily applied to \abap. 
Take $\mathcal{F}$ to be an \abap~framework $\abafp$, let $E$ be its $<$-$\sigma$ extension, and given $\psi \in \cn(E)$, define $\mathcal{F}'$ as follows: 
\begin{itemize}
\item \texttt{STRICT} setting: $\mathcal{F}' = (\mathcal{L}, \mathcal{R} \cup \{ \psi \leftarrow \top \}, \mathcal{A}, \contrary, \leqslant)$;
\item \texttt{ASM} setting: $\mathcal{F}' = (\mathcal{L}, \mathcal{R} \cup \{ \psi \leftarrow \top \}, \mathcal{A} \setminus \{ \psi \}, \contrary, \leqslant')$, 
where $\leqslant'$ is a restriction of $\leqslant$ to $\mathcal{A} \setminus \{ \psi \}$.
\end{itemize}

We can then analyse whether the non-monotonic inference properties in question are satisfied in \abap. 
Trivially, as \abap~is a conservative extension of ABA (cf.~Theorem \ref{theorem:ABA+ extends ABA}), 
properties violated in ABA will remain violated in \abap. 
Therefore, we will focus on those that are satisfied in ABA; in particular, the credulous versions except for \mon~under ideal semantics. 

\begin{example}
As an illustration of the properties, recall Example \ref{example:fundamental}. 
The \abap~framework $\mathcal{F} = (\mathcal{L}, \mathcal{R}', \mathcal{A}, \contrary, \leqslant)$ 
(that satisfies \contraposition) has a unique $<$-$\sigma$ extension $\{ \beta, \gamma \}$ 
with $\cn(\{ \beta, \gamma \}) = \{ \overline{\alpha}, \beta, \gamma \}$. 
\begin{itemize}
\item \texttt{STRICT} setting: take $\overline{\alpha}$ and let 
$\mathcal{F}' = (\mathcal{L}, \mathcal{R} \cup \{ \overline{\alpha} \leftarrow \top \}, \mathcal{A}, \contrary, \leqslant)$. 
Then $\mathcal{F}'$ has a unique $<$-$\sigma$ extension $\{ \beta, \gamma \}$.

\item \texttt{ASM} setting: take $\beta$ and let 
$\mathcal{F}' = (\mathcal{L}, \mathcal{R} \cup \{ \beta \leftarrow \top \}, \mathcal{A} \setminus \{ \beta \}, \contrary, \leqslant')$ with $\alpha <' \gamma$. 
Then $\mathcal{F}'$ likewise has a unique $<$-$\sigma$ extension $\{ \beta, \gamma \}$. 
\end{itemize}
As conclusions of extensions of both $\mathcal{F}$ and $\mathcal{F}'$ are actually the same, 
the credulous versions of the properties are indeed satisfied in both settings. 
\end{example}

In what follows, we assume that a given \abap\ framework $\F = \abafp$ satisfies \emph{\wcp} \cite{Cyras:Toni:2016-arxiv}), which is a weaker form of \contraposition. 

We first show that \abap\ inherits the behaviour from ABA with respect to the non-monotonic inference properties under $<$-stable semantics.

\begin{proposition}
\label{proposition:<-stable satisfies strict}
\F\ fulfils \scutc\ and \smonc\ for $<$-stable semantics.
\end{proposition}

\begin{proof}
Let $\F = \abafp$ be a flat \abap\ framework satisfying \wcp, let $\asmE$ be a $<$-stable extension of $\F$, and let $\psi \in \cn(\asmE) \setminus \A$. 
Define $\F' = (\LL, \R \cup \{ \psi \ot \top \}, \A, \contrary, \leqslant)$. 
$\cn$ and $\cn'$ will denote the conclusion operators of, respectively, $\F$ and $\F'$. 
The $<$-attack relations of $\F$ and $\F'$ will be denoted by $\pattacks$ and $\pattacks'$, respectively. 
We claim that $\asmE$ is a $<$-stable extension of $\F'$. 

Suppose for a contradiction that $\asmE$ is not $<$-conflict-free in $\F'$. 
Then $\asmE$ is not conflict-free in $(\LL, \R \cup \{ \psi \ot \top \}, \A, \contrary)$, by \cite[Theorem 5]{Cyras:Toni:2016-arxiv}. 
But then, as $\psi \in \cn(\asmE)$, $\asmE$ is not conflict-free in $\abaf$ either. 
Hence, by \cite[Theorem 5]{Cyras:Toni:2016-arxiv}, $\asmE$ is not $<$-conflict-free in $\F$, which is a contradiction. 
Thus, $\asmE$ is $<$-conflict-free in $\F'$. 

Now let $\asmb \in \A \setminus \asmE$ be arbitrary. 
We aim to show that $\asmE \pattacks' \{ \asmb \}$. 
To this end, as $\asmE$ is $<$-stable in $\F$, we know that $\asmE \pattacks \{ \asmb \}$. 
We split into cases. 

\begin{itemize}
\item Suppose $\asmE \pattacks \{ \asmb \}$ via normal attack. 
Then $\asmA \vdash^R \contr{\asmb}, \asmA \subseteq \asmE, R \subseteq \R$ and $\forall \asma \in \asmA~~\asma \not< \asmb$. 
If this deduction does not involve $\psi$, then clearly we have $\asmA \pattacks' \{ \asmb \}$ via normal attack. 
Else, we can find $\asmA' \subseteq \asmA$ and $R' \subseteq R \cup \{ \psi \ot \top \}$ such that $\asmA' \vdash^{R'} \contr{\asmb}$, whence clearly $\asmA' \pattacks' \{ \asmb \}$ via normal attack too. 

\item Suppose $\asmE \pattacks \{ \asmb \}$ via reverse attack. 
Then $\{ \asmb \} \vdash \contr{\asme}$ for some $\asme \in \asmE$ such that $\asmb < \asme$. 
Since $\asmb \not\in \asmE$ and $\asmE$ is $<$-conflict-free, this $<$-attack does not involve $\psi$. 
Hence, $\{ \asme \} \pattacks' \{ \asmb \}$ via reverse attack too. 
\end{itemize}
In any event, $\asmE \pattacks' \{ \asmb \}$, as required. 
Therefore, $\asmE$ is $<$-stable in $\F'$. 
\end{proof}

\begin{proposition}
\label{proposition:<-stable satisfies asm}
\F\ fulfils \acutc\ and \amonc\ for $<$-stable semantics.
\end{proposition}

\begin{proof}
Let $\F = \abafp$ be a flat \abap\ framework satisfying \wcp, let $\asmE$ be a $<$-stable extension of $\F$, and let $\psi \in \cn(\asmE) \cap \A$. 
Define $\F' = (\LL, \R \cup \{ \psi \ot \top \}, \A \setminus \{ \psi \}, \contrary, \leqslant')$, 
where $\leqslant'$ is a restriction of $\leqslant$ to $\A \setminus \{ \psi \}$.
$\cn$ and $\cn'$ will denote the conclusion operators of, respectively, $\F$ and $\F'$. 
The $<$-attack relations of $\F$ and $\F'$ will be denoted by $\pattacks$ and $\pattacks'$, respectively. 
We show that $\asmE \setminus \{ \psi \}$ is $<$-stable in $\F'$. 

Suppose for a contradiction that $\asmE \setminus \{ \psi \}$ is not $<$-conflict-free in $\F'$. 
Then $\asmE \setminus \{ \psi \}$ is not conflict-free in $(\LL, \R \cup \{ \psi \ot \top \}, \A \setminus \{ \psi \}, \contrary)$, by \cite[Theorem 5]{Cyras:Toni:2016-arxiv}. 
But then, as $\psi \in \cn(\asmE)$, $\asmE$ is not conflict-free in $\abaf$.  
Hence, by \cite[Theorem 5]{Cyras:Toni:2016-arxiv}, $\asmE$ is not $<$-conflict-free in $\F$, which is a contradiction. 
Thus, $\asmE \setminus \{ \psi \}$ must be $<$-conflict-free in $\F'$. 

Now let $\asmb \in \A \setminus (\asmE \cup \{ \psi \})$ be arbitrary. 
We aim to show that $\asmE \setminus \{ \psi \} \pattacks' \{ \asmb \}$. 
To this end, as $\asmE$ is $<$-stable in $\F$, we know that $\asmE \pattacks \{ \asmb \}$. 

\begin{itemize}
\item Suppose $\asmE \pattacks \{ \asmb \}$ via normal attack. 
Then $\asmA \vdash^R \contr{\asmb}, \asmA \subseteq \asmE, R \subseteq \R$ and $\forall \asma \in \asmA~~\asma \not< \asmb$. 
If $\psi \not\in \asmA$, then we have $\asmA \pattacks' \{ \asmb \}$ via normal attack. 
Else, we find $\asmA \setminus \{ \psi \} \vdash^{R} \contr{\asmb}$, so that $\asmA \setminus \{ \psi \} \pattacks' \{ \asmb \}$ via normal attack. 

\item Suppose $\asmE \pattacks \{ \asmb \}$ via reverse attack. 
Then $\{ \asmb \} \vdash \contr{\asme}$ for some $\asme \in \asmE$ such that $\asmb < \asme$. 
If $\asme \neq \psi$, then this $<$-attack does not involve $\psi$, and so we have $\{ \asme \} \pattacks' \{ \asmb \}$ via reverse attack, where $\{ \asme \} \subseteq \asmE \setminus \{ \psi \}$. 
Else, $\{ \asmb \} \vdash \contr{\psi}$ and $\asmb < \psi$, so \wcp\ guarantees that, in $\F$,  we have $\asmA \vdash \contr{\asmb}$ for some $\asmA \subseteq \{ \psi \}$. 
But then, in $\F'$, we find $\emptyset \vdash \contr{\asmb}$. 
\end{itemize}
In any event, $\asmE \setminus \{ \psi \} \pattacks' \{ \asmb \}$, as required. 
Therefore, $\asmE \setminus \{ \psi \}$ is $<$-stable in $\F'$. 
Finally, note that $\cn(\asmE) = \cn'(\asmE \setminus \{ \psi \})$. 
\end{proof}

In general, \abap\ does not inherit all the properties from ABA. 
In particular, \cut\ and \mon\ can in general be violated in both \texttt{STRICT} and \texttt{ASM} settings under all but $<$-stable semantics. 
The following examples illustrate violations. 

\begin{example}[\smon\ violation]
\label{example:strict mon violation}
Consider $\F = \abafp$ with
\begin{itemize}
\item $\A = \{ \asma, \asmb, \asmp, \asmq, \asme, \asmx \}$, 
\item $\R = \{ \psi \ot \asmp, \asmq, ~~ 
\contr{\asme} \ot \asmb, \asmx, \psi, ~~
\contr{\asma} \ot \asmb, \asmx, \asmp, ~~ 
\contr{\asmb} \ot \asma, \asmx, \asmp, ~~ 
\contr{\asmx} \ot \asmx \}$, 
\item $\asmb < \asma$. 
\end{itemize}

This flat \abap\ framework $\F$ satisfies \wcp. 
It has a unique $<$-grounded/$<$-ideal/$<$-preferred/$<$-complete (but not $<$-stable) extension $\asmE = \{ \asmp, \asmq, \asma, \asme \}$ with $\cn(\asmE) = \{ \asmp, \asmq, \asma, \psi, \asme \}$. 
Note that $\{ \asma \}$ $<$-defends $\{ \asme \}$ from $\{ \asmb, \asmx, \asmp, \asmq \}$ by $<$-attacking the latter via reverse attack, due to the rule $\contr{\asma} \ot \asmb, \asmx, \asmp$ and the preference $\asmb < \asma$. 

Consider $\F' = (\LL, \R \cup \{ \psi \ot \top \}, \A, \contrary, \leqslant)$. 
In this framework, $\{ \asme \}$ is $<$-attacked by the self-$<$-attacking $\{ \asmb, \asmx \}$, and no subset of $\asmE$ can $<$-defend $\{ \asme \}$ against this $<$-attack. 
Indeed, $\F'$ has a unique $<$-grounded/$<$-ideal/$<$-preferred/$<$-complete (but not $<$-stable) extension $\asmE' = \{ \asmp, \asmq, \asma \}$ with $\cn'(\asmE') = \{ \asmp, \asmq, \asma, \psi \} \nsupseteq \cn(\asmE)$. 
(Here and in further examples, $\cn'$ is the conclusion operator of $\F'$.) 
Hence, $\F$ does not fulfil \smon\ under any of the four semantics in question. 
\end{example}

\begin{example}[\amon\ violation]
\label{example:asm mon violation}
Consider $\F$ and $\asmE$ from Example \ref{example:strict mon violation}. 
Let $\F' = (\LL, \R \cup \{ \asma \ot \top \}, \A \setminus \{ \asma \}, \contrary, \emptyset)$. 
In $\F'$, $\{ \asme \}$ is $<$-attacked by the self-$<$-attacking $\{ \asmb, \asmx, \asmp, \asmq \}$, and cannot be $<$-defended by any set not containing $\asmx$. 
Overall, $\F'$ has a unique $<$-grounded/$<$-ideal/$<$-preferred/$<$-complete (but not $<$-stable) extension $\asmE' = \{ \asmp, \asmq \}$ with $\cn'(\asmE') = \{ \asmp, \asmq, \asma, \psi \} \nsupseteq \cn(\asmE)$. 
Hence, $\F$ does not fulfil \amon\ under any of the four semantics in question. 
\end{example}

\begin{example}[\acut\ violation]
\label{example:asm cut violation}
Consider $\F = \abafp$ with
\begin{itemize}
\item $\A = \{ \asma, \asmb, \asmp, \asmq, \asme, \asmx \}$, 
\item $\R = \{ \psi \ot \asmp, \asmq, ~~ 
\contr{\asme} \ot \asmb, \asmx, \psi, ~~
\contr{\asma} \ot \asmb, \asmx, \asmp, ~~ 
\contr{\asmb} \ot \asma, \asmx, \asmp, ~~ 
\contr{\asmx} \ot \asmx, ~~ 
\psi \ot \top \}$, 
\item $\asmb < \asma$. 
\end{itemize}

So $\F$ is simply $\F'$ from Example \ref{example:strict mon violation}. 
It satisfies \wcp\ and we know that it has a unique $<$-grounded/$<$-ideal/$<$-preferred/$<$-complete (but not $<$-stable) extension $\asmE = \{ \asmp, \asmq, \asma \}$ with $\cn(\asmE) = \{ \asmp, \asmq, \asma, \psi \}$. 
Let $\F' = (\LL, \R \cup \{ \asmp \ot \top \}, \A \setminus \{ \asmp \}, \contrary, \leqslant)$. 
In $\F'$, given that $\asmp$ is a fact, the $<$-attacker $\{ \asmb, \asmx \}$ of $\{ \asme \}$ is $<$-attacked by $\{ \asma \}$ via reverse attack. 
Thus, $\{ \asma \}$ $<$-defends $\{ \asme \}$, and so $\asmE' = \{ \asmq, \asma, \asme \}$ with $\cn'(\asmE') = \{ \asmp, \asmq, \asma, \psi, \asme \} \nsubseteq \cn(\asmE)$ is a unique $<$-grounded/$<$-ideal/$<$-preferred/$<$-complete (but not $<$-stable) extension of $\F'$. 
This shows that $\F$ does not fulfil \acut\ under any of the four semantics in question. 
\end{example}

\begin{example}[\scut\ violation]
\label{example:strict cut violation}
Consider $\F = \abafp$ with
\begin{itemize}
\item $\A = \{ \asma, \asmb, \asmp, \asmq, \asme, \asmx \}$, 
\item $\R = \{ \psi \ot \asmp, \asmq, ~~ 
\contr{\asme} \ot \asmb, \asmx, \psi, ~~
\contr{\asma} \ot \asmb, \asmx, y, ~~ 
y \ot \asmp, ~~
\contr{\asmb} \ot \asma, \asmx, \asmp, ~~ 
\contr{\asmx} \ot \asmx, ~~ 
\contr{\asmb} \ot \asma, \asmx, ~~
\psi \ot \top \}$, 
\item $\asmb < \asma$. 
\end{itemize}
(So, in contrast to the framework from Example \ref{example:strict mon violation}, there is an intermediate non-assumption $y$ deducible from $\{ \asmp \}$ and replacing $\asmp$ in the rule $\contr{\asma} \ot \asmb, \asmx, \asmp$; 
we also have $\psi$ as a fact, and the rule $\contr{\asmb} \ot \asma, \asmx$ will be needed for Weak Contraposition in the framework $\F'$ after the change.)

$\F$ satisfies \wcp\ and has a unique $<$-grounded/$<$-ideal/$<$-preferred/$<$-complete (but not $<$-stable) extension $\asmE = \{ \asmp, \asmq, \asma \}$ with $\cn(\asmE) = \{ \asmp, \asmq, \asma, \psi, y \}$. 
Let $\F' = (\LL, \R \cup \{ y \ot \top \}, \A, \contrary, \leqslant)$. 
(Since in $\F'$ we have the deduction $\{ \asmb, \asmx \} \vdash^{\{ \contr{\asma} \ot \asmb, \asmx \}} \contr{\asma}$ with $\asmb < \asma$, the rule $\contr{\asmb} \ot \asma, \asmx$ guarantees that $\F'$ satisfies \wcp.)
Similarly to Example \ref{example:asm cut violation}, $\{ \asma \}$ $<$-defends $\{ \asme \}$, and $\F'$ has a unique $<$-grounded/$<$-ideal/$<$-preferred/$<$-complete (but not $<$-stable) extension $\asmE' = \{ \asmp, \asmq, \asma, \asme \}$ with $\cn'(\asmE') = \{ \asmp, \asmq, \asma, \psi, y, \asme \} \nsubseteq \cn(\asmE)$. 
Hence, $\F$ does not fulfil \scut\ under any of the four semantics in question. 
\end{example}

Table 3 summarizes this section's results 
(sceptical and credulous versions coincide under $<$-grounded and $<$-ideal semantics; 
for other semantics the credulous version is indicated in parentheses.)

\begin{table}[h]
\begin{center}
\begin{tabular}{| >{\centering}m{2cm}<{\centering} | c | c | c | c | c | c |}
\hline
Property & $<$-g.\!\! & $<$-id.\!\! & $<$-stb.\!\! & $<$-pr.\!\! & $<$-cpl.\!\! \\ \hline
\texttt{STRICT/ASM} \cut & \sffamily{X} & \sffamily{X} & \sffamily{X} ($\checkmark$) & \sffamily{X} (\sffamily{X}) & \sffamily{X} (\sffamily{X}) \\ \hline

\texttt{STRICT/ASM} \mon & \sffamily{X} & \sffamily{X} & \sffamily{X} ($\checkmark$) & \sffamily{X} (\sffamily{X}) & \sffamily{X} (\sffamily{X}) \\ \hline
\end{tabular}
\end{center} 
\caption{(\texttt{STRICT} and \texttt{ASM}) \cut~and \mon~for \abap}
\end{table}

\section{Related and Future Work}
\label{sec:Related Work}

The principle of Contraposition of (strict) rules (see e.g.~\cite{Caminada:Amgoud:2007,Modgil:Prakken:2013}) is notably employed in the well studied structured argumentation formalism \aspicp~\cite{Modgil:Prakken:2013,Modgil:Prakken:2014}. 
The principle as such is also inherently present in classical logic-based approaches to structured argumentation such as \cite{Gorogiannis:Hunter:2011,Besnard:Hunter:2014}. 
Similarly as in \aspicp, \abap~utilizes Contraposition to ensure the Fundamental Lemma (cf.~Lemma \ref{lemma:Fundamental}).
As a consequence, Contraposition paves way to satisfaction of desirable properties of \abap~semantics, as well as preference handling as discussed in Section \ref{sec:Preference Handling}. 
Whether \contraposition~can be relaxed for \abap~to obtain the same results is a line of future research. 
So far we know only that a particular relaxation, namely \wcp~\cite{Cyras:Toni:2016-arxiv}, changes the behaviour of \abap\ with respect to non-monotonic reasoning properties, as discussed in Section \ref{sec:NMR}.

The preference handling principle discussed in Section \ref{subsec:Principle I} was originally proposed, along with some other properties, by \cite{Brewka:Eiter:1999} for answer set programming (ASP) with preferences. 
To the best of our knowledge, reformulation of Principle I for \abap~is the first application of this principle to argumentation with preferences. 
Building on \cite{Brewka:Eiter:1999}, \cite{Simko:2014} discussed an extended set of principles for ASP with preferences, most of which focus on preferences over rules. 
Whether those principles can be applied to \abap~is a future work direction. 

Regarding preference handling in argumentation, along with \maximal~discussed in Section \ref{subsec:Maximal Elements}, \cite{Amgoud:Vesic:2014} suggested several arguably desirable properties of argumentation with preferences. 
Those properties are exhibited in \abap~as Proposition \ref{proposition:<-conflict-free iff conflict-free} and Theorem \ref{theorem:ABA+ extends ABA}. 
Referring to those properties, \cite{Brewka:Truszczynski:Woltran:2010} also hinted at other properties regarding selection among extensions, as possible principles of preference handling in argumentation. 
Relating those principles to \abap~is left for future work. 

In terms on non-monotonic reasoning properties, Cautious Monotonicity and Cumulative Transitivity (studied in Section \ref{sec:NMR}) 
are traced to \cite{Makinson:1988,Kraus:Lehmann:Magidor:1990} 
and fall into the well studied area of analysing non-monotonic reasoning with respect to information change (cf.~\cite{Rott:2001}). 
In argumentation setting, the latter is also known as \emph{argumentation dynamics}, and has recently been a topic of interest in the argumentation community (see e.g.~\cite{Cayrol:Saint-Cyr:Lagasquie-Schiex:2010,Falappa:Garcia:Kern-Isberner:Simari:2011,Baroni:Boella:Cerutti:Giacomin:Torre:Villata:2014,Coste-Marquis:Konieczny:Mailly:Marquis:2014,Booth:Gabbay:Kaci:Rienstra:Torre:2014,Diller:Haret:Linsbichler:Rummele:Woltran:2015,Baumann:Brewka:2015}). 
In particular, non-monotonic inference properties were investigated in \cite{Hunter:2010} with respect to argument--claim entailment in logic-based argumentation systems; 
in \cite{Cyras:Toni:2015} for ABA; 
and with regards to \aspicp-type-of argumentation systems in \cite{Dung:2016}. 
Only the latter of the three works concerns argumentation with preferences. 
In addition to considering different structured argumentation setting and different preference handling mechanisms, 
it diverges from our analysis in Section \ref{sec:NMR} in that
\cite{Dung:2016} regards Cumulative Transitivity plus Cautious Monotonicity as a single property of Cumulativity and studies it only for stable and complete semantics. 
Other argumentation-related properties from \cite{Dung:2016} will be studied for \abap~in the future. 

Several other topics of interest are left for future work. 
For instance, integrating \emph{dynamic preferences} (see e.g.~\cite{Zhang:Foo:1997,Prakken:Sartor:1999,Brewka:Woltran:2010}) within \abap~and studying their interaction with the properties of preference handling as well as of non-monotonic inference. 
Also, relation of \abap~to Logic Programming with preferences (e.g.~\cite{Sakama:Inoue:1996,Zhang:Foo:1997,Brewka:Eiter:1999}) and non-monotonic reasoning formalisms equipped with preferences in general (e.g.~\cite{Brewka:1989,Baader:Hollunder:1995,Rintanen:1998,Brewka:Eiter:2000,Delgrande:Schaub:2000,Stolzenburg:Garcia:Chesnevar:Simari:2003,Kakas:Moraitis:2003}) is left for future research. 

There are as well numerous approaches to integrating reasoning with preferences within argumentation, e.g.~\cite{Amgoud:Cayrol:2002,Bench-Capon:2003,Kaci:Torre:2008,Modgil:2009,Modgil:Prakken:2010,Baroni:Cerutti:Giacomin:Guida:2011,Dunne:Hunter:McBurney:Parsons:Wooldridge:2011,Brewka:Ellmauthaler:Strass:Wallner:Woltran:2013,Amgoud:Vesic:2014,Besnard:Hunter:2014,Garcia:Simari:2014,Wakaki:2014,Modgil:Prakken:2013,Modgil:Prakken:2014,Dung:2016}. 
It would be interesting to study these formalisms with respect to the properties considered in this paper, where it has not already been done. 
We leave this as future work.

\section{Conclusions}
\label{sec:Conclusions}

We investigated various properties of a recently proposed non-monotonic reasoning formalism \abap~\cite{Cyras:Toni:2016-KR} that deals with preferences in structured argumentation. 
In particular, we first established that assuming the principle of Contraposition (see e.g.~\cite{Modgil:Prakken:2013}), 
\abap~semantics exhibit desirable properties akin to those of other existing argumentation formalisms, such as \cite{Dung:1995}. 
We then showed that \abap~satisfies some (arguably) desirable principles of preference handling in argumentation and non-monotonic reasoning, e.g.~\cite{Brewka:Eiter:1999}. 
Finally, we analysed non-monotonic inference properties (as in \cite{Cyras:Toni:2015}) of \abap~under various semantics. 
We believe our work contributes to the understanding of preferences within argumentation in particular, and in non-monotonic reasoning at large.

\bibliographystyle{aaai}
\bibliography{../../../Readings/library}

\end{document}